\documentclass{article}

\PassOptionsToPackage{numbers, compress}{natbib}

\usepackage[final]{neurips_2021}

\usepackage{microtype}
\usepackage{xcolor}
\usepackage{booktabs} 
\usepackage{url}  
\usepackage{helvet} 
\usepackage{courier}  
\usepackage{graphicx} 
\usepackage{natbib}  
\usepackage{caption} 
\usepackage{subcaption}
\usepackage{amsthm}
\usepackage{wrapfig}

\usepackage{color}
\usepackage{comment}
\usepackage{amsmath}
\usepackage{amsthm} 
\usepackage{amssymb}
\usepackage{bbm}
\usepackage{mathbbol} 
\usepackage{soul}
\usepackage{stackengine}
\usepackage{mathtools} 
\usepackage[ruled]{algorithm2e}
\usepackage{algorithmic}

\newcommand{\argmin}{\mathop{\mathrm{argmin}}}

\usepackage{bm}

\newcommand{\Expect}{\mathbb{E}}
\newcommand{\Prob}{\mathbb{P}}

\DeclareSymbolFontAlphabet{\mathbb}{AMSb}
\DeclareSymbolFontAlphabet{\mathbbl}{bbold}

\newcommand{\reals}{\mathbb{R}}

\newcommand{\regret}{\bar R_K}

\newcommand{\rlalg}{\text{Alg}_\pi}
\newcommand{\costalg}{\text{Alg}_\lambda}
\newcommand{\piopt}{\tilde \pi _k}
\newcommand{\popt}{\tilde P _k}
\newcommand{\ropt}{\tilde r _k}
\newcommand{\ie}{{\it i.e.}}
\newcommand{\optset}{\mathcal{P}_k}

\newcommand{\dpiz}{\ensuremath{d^z_{\pi}}}
\newcommand{\dpik}{\ensuremath{d^k_{\pi}}}
\newcommand{\dpij}{\ensuremath{d^j_{\pi}}}
\newcommand{\dpikm}{\ensuremath{d^{k-1}_{\pi}}}
\newcommand{\dpione}{\ensuremath{d^1_{\pi}}}

\newcommand{\dpikp}{\ensuremath{d^{k+1}_{\pi}}}
\newcommand{\bdpik}{\ensuremath{\bar{d}^k_{\pi}}}
\newcommand{\bdpikp}{\ensuremath{\bar{d}^{k+1}_{\pi}}}
\newcommand{\bdpikm}{\ensuremath{\bar{d}^{k-1}_{\pi}}}

\newcommand{\bdpi}{\ensuremath{\bar{d}_{\pi}}}
\newcommand{\bc}{\ensuremath{\bar{\lambda}}}
\newcommand{\cost}{\ensuremath{\lambda}}
\newcommand{\costset}{\ensuremath{\Lambda}}
\newcommand{\dpi}{\ensuremath{d_{\pi}}}

\newcommand{\g}{\mathcal{L}}

\newcommand{\eg}{{\it eg}}

\usepackage[utf8]{inputenc} 
\usepackage[T1]{fontenc}    
\usepackage{booktabs}       
\usepackage{amsfonts}       
\usepackage{nicefrac}       
\usepackage{microtype}      
\usepackage{wrapfig}

\usepackage{natbib}
\usepackage{graphicx}
\usepackage{microtype}
\usepackage{amsmath}
\usepackage{amssymb}

\usepackage{booktabs} 
\usepackage{thmtools}
\usepackage{thm-restate}
\newtheorem{lemma}{Lemma}
\newtheorem*{lemma*}{Lemma}

\newtheorem{theorem}{Theorem}
\newtheorem{definition}{Definition}

\newtheorem{remark}{Remark}

\usepackage{hyperref}

\usepackage[capitalize]{cleveref}

\author{Tom Zahavy \\ DeepMind, London\\tomzahavy@deepmind.com \And Brendan O'Donoghue \\ DeepMind, London\\bodonoghue@deepmind.com\\ \AND Guillaume Desjardins \\ DeepMind, London\\gdesjardins@deepmind.com \\ \And Satinder Singh \\ DeepMind, London \\ baveja@deepmind.com}
\title{Reward is Enough for Convex MDPs}

\begin{document}

\maketitle

\begin{abstract}
Maximising a cumulative reward function that is Markov and stationary, \ie, defined over state-action pairs and independent of time, is sufficient to capture many kinds of goals in a Markov decision process (MDP). However, not all goals can be captured in this manner. In this paper we study convex MDPs in which goals are expressed as convex functions of the stationary distribution and show that they cannot be formulated using stationary reward functions. Convex MDPs generalize the standard reinforcement learning (RL) problem formulation to a larger framework that includes many supervised and unsupervised RL problems, such as apprenticeship learning, constrained MDPs, and so-called `pure exploration'. Our approach is to reformulate the convex MDP problem as a min-max game involving policy and cost (negative reward) `players', using Fenchel duality. We propose a meta-algorithm for solving this problem and show that it unifies many existing algorithms in the literature. 
\end{abstract}

\section{Introduction}
\label{sec:problem}
In reinforcement learning (RL), an agent learns how to map situations to actions so
as to maximize a cumulative scalar reward signal. The learner is not told which actions to take, but instead must discover which actions lead to the most reward \citep{sutton2018reinforcement}. Mathematically, the RL problem can be written as finding a policy whose state occupancy has the largest inner product with a reward vector \citep{puterman2014markov}, \ie, the goal of the agent is to solve
\begin{align}
\label{eq:rl-prob}
\text{RL:}\quad & \max_{\dpi \in \mathcal{K}} \sum_{s,a} r(s,a) \dpi(s,a),\\
\shortintertext{where $\dpi$ is the state-action stationary distribution induced by policy $\pi$ and $\mathcal{K}$ is the set of admissible stationary distributions (see \cref{def:polytope}).
A significant body of work is dedicated to solving the RL problem efficiently in challenging domains \cite{mnih2015human,silver2017mastering}.
However, not all decision making problems of interest take this form. In particular we consider the more general \emph{convex} MDP problem,}
\label{eq:cvx-prob}
\text{Convex MDP:}\quad & \min_{\dpi \in \mathcal{K}}f(\dpi),
\end{align}
where $f : \mathcal{K} \rightarrow \reals$ is a convex function. Sequential decision making problems that take this form include Apprenticeship Learning (AL), pure exploration, and constrained MDPs, among others; see Table \ref{table:c_MDPs}. In this paper we prove the following claim: 
\begin{center}
\emph{We can solve \cref{eq:cvx-prob} by using any algorithm that solves \cref{eq:rl-prob} as a subroutine}.
\end{center}

In other words, any algorithm that solves the standard RL problem can be used to solve the more general convex MDP problem. More specifically, we make the following contributions.

\textbf{Firstly}, we adapt the meta-algorithm of Abernethy and Wang \citep{abernethy2017frankwolfe} for solving \cref{eq:cvx-prob}. The key idea is to use Fenchel duality to convert the convex MDP problem into a two-player zero-sum game between the agent (henceforth, \emph{policy player}) and an adversary that produces rewards (henceforth, \emph{cost player}) that the agent must maximize \citep{abernethy2017frankwolfe, agrawal2014fast}. From the agent's point of view, the game is bilinear, and so for fixed rewards produced by the adversary the problem reduces to the standard RL problem with non-stationary reward (\cref{fig:convex_MDP}). 

\begin{wrapfigure}{r}{0.5\textwidth}
\centering
\includegraphics[width=.8\linewidth]{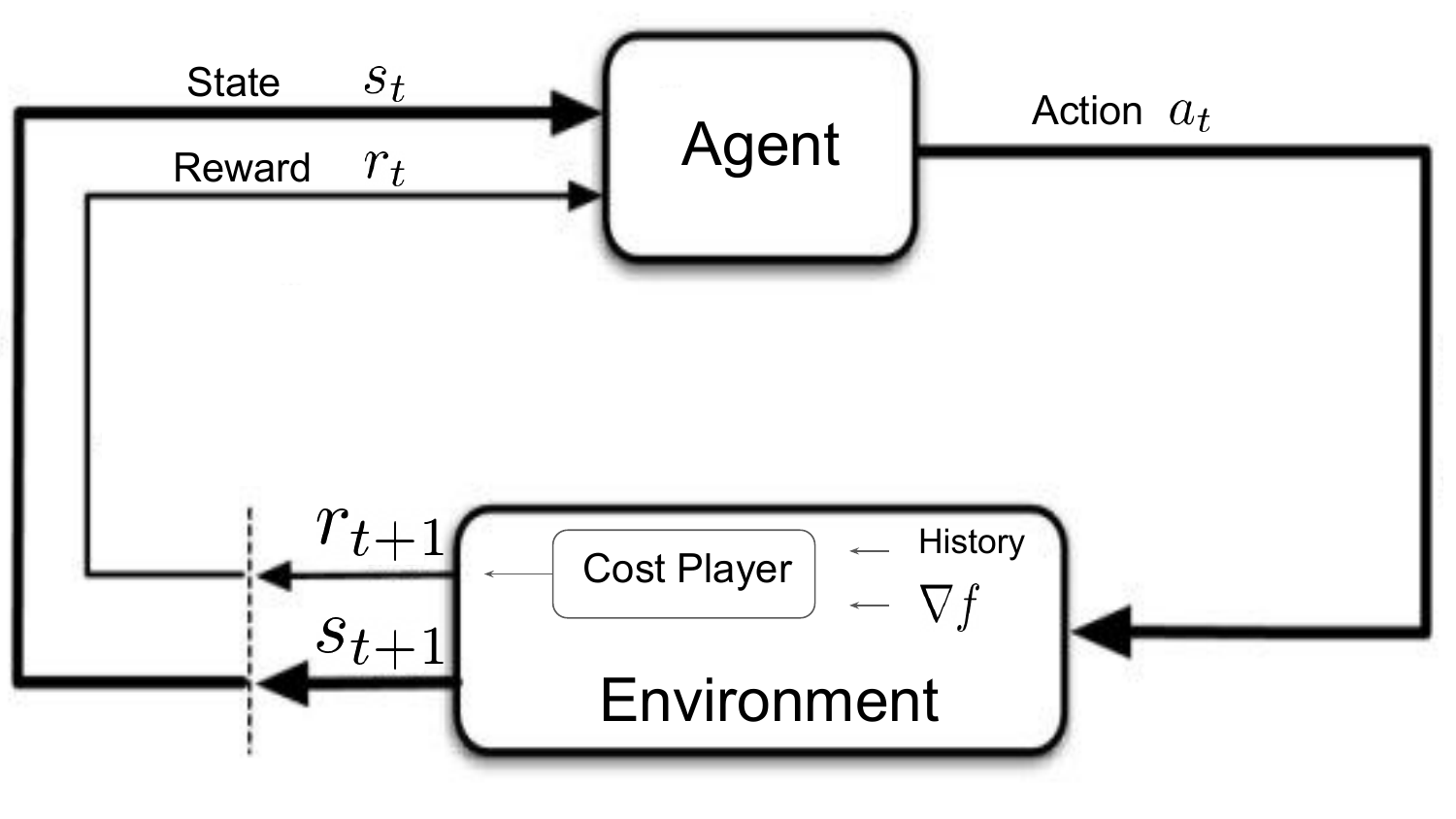}
\caption{Convex MDP as an RL problem}
\label{fig:convex_MDP}
\end{wrapfigure} 
\textbf{Secondly}, 
we propose a sample efficient policy player that uses a standard RL algorithm (\eg, \citep{jaksch2010near,shani2020optimistic}), and computes an optimistic policy with respect to the non-stationary reward at each iteration. In other words, we use algorithms that were developed to achieve low regret in the standard RL setup, to achieve low regret as policy players in the min-max game we formulate to solve the convex MDP. 
Our main result is that the average of the policies produced by the policy player converges to a solution to the convex MDP problem (\cref{eq:cvx-prob}).
Inspired by this principle, we also propose a recipe for using deep-RL (DRL) agents to solve convex MDPs heuristically: provide the agent non-stationary rewards from the cost player. We explore this principle in our experiments.

\textbf{Finally}, we show that choosing specific algorithms for the policy and cost players unifies several disparate branches of RL problems, such as apprenticeship learning, constrained MDPs, and pure exploration into a single framework, as we summarize in \cref{table:c_MDPs}.

\begingroup
\begin{table}[h]
\resizebox{\textwidth}{!}{
    \centering
    \begin{tabular}{llll}
        \toprule
        Convex objective $f$ & Cost player & Policy player & Application  \\
        \midrule
        $\cost \cdot \dpi$ &  FTL & RL & (Standard) RL with $-\cost$ as stationary reward function\\
        $||\dpi- d_E||_2^2$ &  FTL &  Best response & Apprenticeship learning (AL) \citep{abbeel2004apprenticeship,zahavy2019apprenticeship}\\
        $\dpi \cdot \log(\dpi)$ &  FTL &  Best response  & Pure exploration$^*$ \citep{hazan2019provably}\\
        $||\dpi- d_E||_\infty$ & OMD &  Best response & AL \citep{syed2008apprenticeship,syed2008game} \\
        $\mathbb{E}_c \left[\lambda c \cdot (\dpi - d_E(c))\right]^\dagger$ & OMD &  Best response & Inverse RL in contextual MDPs \citep{belogolovsky2019inverse}\\
        $\cost_1 \cdot \dpi, \enspace \text{s.t.} \enspace \cost_2 \cdot \dpi \le c$  & OMD & RL & Constrained MDPs \citep{altman1999constrained,cmdpblog,borkar2005actor,tessler2018reward,efroni2020exploration,calian2021balancing,bhatnagar2012online}\\
        $\text{dist}(\dpi, C)^{\dag\dag}$ & OMD &  Best response & Feasibility of convex-constrained MDPs \citep{miryoosefi2019reinforcement}\\
        $\min_{\cost_1,\ldots,\cost_k} \dpik\cdot\cost_k$&OMD & RL & Adversarial Markov Decision Processes \citep{rosenberg2019online}\\
        $\max_{\cost\in \costset} \enspace \cost \cdot (\dpi - d_E)$& OMD & RL &  Online AL \citep{shani2021online},Wasserstein GAIL \citep{xiao2019wasserstein,zhang2020wasserstein}\\ 
         $\text{KL}(\dpi || d_E)$ &  FTL & RL &GAIL \citep{ho2016generative}, state marginal matching \citep{lee2019efficient},\\
         $-\mathbb{E}_z \text{KL}(\dpiz || \mathbb{E}_k \dpik) ^\ddag$ & FTL& RL & Diverse skill discovery  \citep{gregor2016variational,eysenbach2018diversity, hausman2018learning,florensa2016stochastic,tirumala2020behavior,achiam2018variational}\\ 
        \bottomrule\\
  \end{tabular}
}
\caption{Instances of \cref{alg:meta} in various convex MDPs. $^*$ as well as other KL divergences. $^\dagger$ $c$ is a context variable. $^{\dag\dag}$ $C$ is a convex set. $^\ddag$ $f$ is concave.  See Sections \ref{sec:players} $\&$ \ref{sec:examples} for more details. }
\label{table:c_MDPs}
\end{table}
\endgroup


\section{Reinforcement Learning Preliminaries}
In RL an agent interacts with an environment over a number of time steps and seeks to maximize its cumulative reward. We consider two cases, the average reward case and the discounted case.
The Markov decision process (MDP) is defined by the tuple $(S,A,P,R)$ for the average reward case and by the tuple $(S, A,P,R, \gamma, d_0)$ for the discounted case. We assume an infinite horizon, finite state-action problem where initially, the state of the agent is sampled according to $s_0 \sim d_0$, then at
each time $t$ the agent is in state $s_t \in S$, selects action $a_t \in A$ according to some policy  $\pi(s_t, \cdot)$, receives reward $r_t \sim R(s_t, a_t)$ and transitions to new state $s_{t+1} \in S$ according to the probability distribution $P(\cdot, s_t, a_t)$. The two performance metrics we consider are given by
\begin{equation}
\label{eq:agent_obj}
J^\mathrm{avg}_\pi = \lim_{T\rightarrow \infty} \frac{1}{T} \Expect\sum_{t=1}^T r_t, \quad
J^\mathrm{\gamma}_\pi =(1-\gamma)\Expect \sum_{t=1}^\infty \gamma^t r_t,
\end{equation}
for the average reward case and discounted case respectively.
The goal of the agent is to find a policy that maximizes $J^\mathrm{avg}_\pi$ or $J^\mathrm{\gamma}_\pi$. Any stationary policy $\pi$ induces a \emph{state-action occupancy measure} $\dpi$, which measures how often the agent visits each state-action when following $\pi$. 
Let $\Prob_\pi(s_t=\cdot)$ be the probability measure over states at time $t$ under policy $\pi$, then
\[
\dpi^{\mathrm{avg}}(s,a) = \lim_{T\rightarrow \infty}\frac{1}{T}\Expect\sum_{t=1}^T \Prob_\pi(s_t = s) \pi(s,a), \enspace \dpi^{\mathrm{\gamma}}(s,a) =  (1-\gamma)\Expect\sum_{t=1}^\infty \gamma^t \Prob_\pi(s_t = s) \pi(s,a),
\]
for the average reward case and the discounted case respectively.
With these, we can rewrite the RL objective in \cref{eq:agent_obj} in terms of the occupancy measure using the following well-known result, which for completeness we prove in \cref{sec:propone}.
\begin{restatable}{prop}{propone}
\label{propone}
For both the average and the discounted case, the agent objective function \cref{eq:agent_obj} can be written in terms of the occupancy
measure as $J_\pi = \sum_{s,a} r(s,a) \dpi(s,a)$.
\end{restatable}
Given an occupancy measure it is possible to recover the policy by setting $\pi(s,a) = \dpi(s,a) / \sum_a \dpi(s,a)$ if $\sum_a \dpi(s,a)>0$, and $\pi(s,a) = 1/|A|$ otherwise. Accordingly, in this paper we shall formulate the RL problem using the state-action occupancy measure, and both the standard RL problem (\cref{eq:rl-prob}) and the convex MDP problem (\cref{eq:cvx-prob}) are convex optimization problems in variable $d_\pi$. 
For the purposes of this manuscript we do not make a distinction between the average and discounted settings, other than through the convex polytopes of feasible occupancy measures, which we define next.
\begin{definition} [State-action occupancy's polytope \cite{puterman2014markov}]
 \label{def:polytope}
 For the average reward case the set of admissible state-action occupancies is
\begin{align*}
 \mathcal{K}_\mathrm{avg} &= \Huge\{\dpi \mid \dpi \geq 0,\ \sum_{s,a}\dpi(s,a) = 1,\ \sum_a \dpi(s,a) = \sum_{s^\prime, a^\prime} P(s, s^\prime, a^\prime) \dpi(s^\prime, a^\prime) \  \forall s \in S\Huge\},\\
\shortintertext{and for the discounted case it is given by}
 \mathcal{K}_\gamma &= \Huge\{\dpi \mid \dpi \geq 0,\ \sum_a \dpi(s,a) = (1-\gamma)d_0(s) + \gamma\sum_{s^\prime, a^\prime} P(s, s^\prime, a^\prime) \dpi(s^\prime, a^\prime) \  \forall s \in S \Huge\}.
\end{align*}
\end{definition}
We note that being a polytope implies that $\mathcal{K}$ is a convex and compact set.

 
The convex MDP problem is defined for the tuple $(S, A, P, f)$ in the average cost case and $(S, A, P, f, \gamma, d_0)$ in the discounted case. This tuple is defining a state-action occupancy's polytope $\mathcal{K}$ (\cref{def:polytope}), and the problem is to find a policy $\pi$ whose state occupancy $\dpi$ is in this polytope and minimizes the function $f$ (\cref{eq:cvx-prob}).

\section{A Meta-Algorithm for Solving Convex MDPs via RL}
\label{sec:problem_formulation}
To solve the convex MDP problem (\cref{eq:cvx-prob}) we need to find an occupancy measure $\dpi$ (and associated policy) that minimizes the function $f$. Since both $f: \mathcal{K} \rightarrow \reals$ and the set $\mathcal{K}$ are convex this is a convex optimization problem. However, it is a challenging one due to the nature of learning about the environment through stochastic interactions. In this section we show how to reformulate the convex MDP problem (\cref{eq:cvx-prob}) so that standard RL algorithms can be used to solve it, allowing us to harness decades of work on solving vanilla RL problems. To do that we will need the following definition.
\begin{definition}[Fenchel conjugate]
\label{def:conj}
For a function $f:\reals^n\rightarrow \reals \cup \{-\infty, \infty\}$, its Fenchel conjugate is denoted $f^*:\reals^n \rightarrow \reals \cup \{-\infty, \infty\}$ and defined as $f^*(x):= \sup_{y} x\cdot y - f(y).$ 
\end{definition}
\begin{remark}
\label{remark}
The Fenchel conjugate function $f^*$ is always convex (when it exists) even if $f$ is not. Furthermore, the biconjugate $f^{**} \coloneqq (f^*)^*$ equals $f$ if and only if $f$ is convex and lower semi-continuous.
\end{remark}

Using this we can rewrite the convex MDP problem (\cref{eq:cvx-prob}) as
\begin{equation}
\label{eq:cvx-prob_fenchel_dual}
f^{\text{OPT}} = \min_{\dpi \in \mathcal{K}}f(\dpi) = \min_{\dpi \in \mathcal{K}} \max_{\cost \in \Lambda} \left( \cost\cdot \dpi - f^*(\cost) \right)
= \max_{\cost \in \Lambda} \min_{\dpi \in \mathcal{K}} \left( \cost\cdot \dpi - f^*(\cost) \right)
\end{equation}
where $\Lambda$ is the closure of (sub-)gradient space $\{\partial f(\dpi)| \dpi \in \mathcal{K} \}$, which is a convex set \citep[Theorem 4]{abernethy2017frankwolfe}. As both sets are convex, this is a convex-concave saddle-point problem and a zero-sum two-player game \cite{osborne1994course, o2020stochastic}, and we were able to swap the order of minimization and maximization using the minimax theorem \cite{neumann1928theorie}. 




With this we define the Lagrangian as $\g(\dpi, \cost) \coloneqq \cost\cdot \dpi - f^*(\cost)$.
For a fixed $\cost\in\Lambda$, minimizing the Lagrangian is a standard RL problem of the form of \cref{eq:rl-prob}, \ie, equivalent to maximizing a reward $r=-\lambda$. Thus, one might hope that by producing an optimal dual variable $\cost^\star$ we could simply solve $\dpi^\star = \argmin_{\dpi \in \mathcal{K}} \g(\cdot, \cost^\star)$ for the optimal occupancy measure. However, the next lemma states that this is not possible in general.
\begin{lemma}
\label{lemma:reward_enough}
There exists an MDP $M$ and convex function $f$ for which there is no stationary reward $r \in \reals^{S \times A}$
such that $\arg\max _{\dpi \in \mathcal{K}} \dpi\cdot r = \arg\min_{\dpi \in \mathcal{K}} f(\dpi)$. 
\end{lemma}
To see this note that for any reward $r$ there is a deterministic policy that optimizes the reward \citep{puterman2014markov}, but for some choices of $f$ no deterministic policy is optimal, \eg, when $f$ is the negative entropy function.
This result tells us that even if we have access to an optimal dual-variable we cannot simply use it to recover the stationary distribution that solves the convex MDP problem in general. 

To overcome this issue we develop an algorithm that generates 
a \emph{sequence} of 
policies $\{\pi^k\}_{k \in \mathbb{N}}$
such that the average converges to an optimal policy for \cref{eq:cvx-prob}, \ie, $(1/K)\sum_{k=1}^K \dpik \rightarrow \dpi^\star \in \arg\min_{\dpi \in \mathcal{K}} f(\dpi)$. 
The algorithm we develop is described in \cref{alg:meta} and is adapted from the meta-algorithm described in \citet{abernethy2017frankwolfe}. It is referred to as a \emph{meta-algorithm} since it
relies on supplied sub-routine algorithms $\rlalg$ and $\costalg$. The reinforcement learning algorithm $\rlalg$ takes as input a reward vector and returns a state-action occupancy measure $\dpi$. The cost algorithm $\costalg$ can be a more general function of the entire history. We discuss concrete examples of $\rlalg$ and $\costalg$ in \cref{sec:players}. 
\begin{algorithm}[H]
\begin{algorithmic}[1]
\STATE \textbf{Input:} convex-concave payoff $\g : \mathcal{K} \times \costset \rightarrow \mathcal{R},$ algorithms $\costalg,\rlalg$, $K \in \mathbb{N}$
\FOR{$k = 1,\ldots,K$}
  \STATE $\cost^k =  \costalg(\dpione, \ldots, \dpikm; \g)$
  \STATE $\dpik = \rlalg(-\cost^k)$\quad 
\ENDFOR
\STATE Return $\bdpi^K = \frac{1}{K}\sum_{k=1}^K \dpik ,\bc^K = \frac{1}{K}\sum_{k=1}^K \cost^k$
\end{algorithmic}
\caption{meta-algorithm for convex MDPs }
\label{alg:meta}
\end{algorithm}

In order to analyze this algorithm we will need a small detour into online convex optimization (OCO). In OCO, a learner is presented with a sequence of $K$ convex loss functions $\ell_1, \ell_2, \ldots , \ell_K : \mathcal{K} \rightarrow \mathbb{R}$ and at each round $k$ must select a point $x_k \in \mathcal{K}$ after which it suffers a loss of $\ell_k(x_k)$. At time period $k$ the learner is assumed to have perfect knowledge of the loss functions $\ell_1, \ldots, \ell_{k-1}$. The learner wants to minimize its \emph{average regret}, defined as
\[
\regret \coloneqq \frac{1}{K} \left(\sum^K_{k=1} \ell_k(x_k) - \min_{x\in\mathcal{K}}\sum^K_{k=1}\ell_k(x)\right).
\]
In the context of convex reinforcement learning and meta-algorithm \ref{alg:meta}, the loss functions for the cost player  are $\ell_\mathrm{\cost}^k = -\g(\cdot, \cost^k)$, and for the policy player are $\ell_\mathrm{\pi}^k = \g(\dpik, \cdot)$, with associated average regrets $\regret^\mathrm{\pi}$ and $\regret^\mathrm{\cost}$. This brings us to the following theorem.
\begin{theorem}[Theorem 2, \citep{abernethy2017frankwolfe}]
\label{thm:meta}
Assume that $\rlalg$ and $\costalg$ have guaranteed average regret bounded as $\regret^\mathrm{\pi} \leq \epsilon_K$ and $\regret^\mathrm{\cost} \leq \delta_K$, respectively.
Then \cref{alg:meta} outputs $\bdpi^K$ and $\bc^K$ satisfying $\min_{\dpi \in\mathcal{K}} \g(\dpi, \bc^K) \ge f^{\text{OPT}} -\epsilon_K - \delta_K$  and $\max_{\cost\in\costset}\g(\bdpi^K, \cost) \le f^{\text{OPT}} +\epsilon_K + \delta_K$.
\end{theorem}
This theorem tells us that so long as the RL algorithm we employ has guaranteed low-regret, and assuming we choose a reasonable low-regret algorithm for deciding the costs, then the meta-algorithm will produce a solution
to the convex MDP problem (\cref{eq:cvx-prob}) to any desired tolerance, this is because
$f^{\text{OPT}} \leq f(\bdpi^K) = \max_{\lambda} \g(\bdpi^K, \lambda) \leq f^{\text{OPT}} + \epsilon_K + \delta_K$. For example, we shall later present algorithms that have regret bounded as $\epsilon_K = \delta_K \leq O(1 / \sqrt{K})$, in which case
we have
\begin{equation}
\label{thm:meta2}
    f(\bdpi^K) - f^{\text{OPT}} \leq O(1/\sqrt{K}).
\end{equation}

\textbf{Non-Convex $f$.}  
\cref{remark} implies that the game $\max_{\cost \in \Lambda} \min_{\dpi \in \mathcal{K}} \left( \cost\cdot \dpi - f^*(\cost) \right)$ is concave-convex for any function $f$, so we can solve it with \cref{alg:meta}, even for a non-convex $f$. From weak duality  the value of the Lagrangian on the output of \cref{alg:meta}, $\g(\bdpi,\bc$), is a lower bound on the optimal solution $f^{\text{OPT}}$. In addition, since $f(\dpi)$ is always an upper bound on $f^{\text{OPT}}$ we have both an upper bound and a lower bound on the optimal value: $\g(\bdpi,\bc) \le f^{\text{OPT}} \le f(\bdpi)$. 



\section{Policy and Cost Players for Convex MDPs}
\label{sec:players}
In this section we present several algorithms for the policy and cost players that can be used in \cref{alg:meta}. Any combination of these algorithms is valid and will come with different practical and theoretical performance. In \cref{sec:examples} we show that several well known methods in the literature correspond to particular choices of cost and policy players and so fall under our framework.

In addition, in this section we assume that $$\cost_\mathrm{max} = \max_{\lambda \in \Lambda} \max_{s,a} |\lambda (s,a)| < \infty,$$ which holds when the set $\Lambda$ is compact. One way to guarantee that $\Lambda$ is compact is to consider functions $f$ with Lipschitz continuous gradients (which implies bounded gradients since the set $\mathcal{K}$ is compact). For simplicity, we further assume that $\cost_\mathrm{max}\le1.$ By making this assumption we assure that the non stationary rewards produced by the cost player are bounded by $1$ as is usually done in RL.



\subsection{Cost Player}

\textbf{Follow the Leader (FTL)} is a classic OCO algorithm that selects $\cost_k$ to be the best point in hindsight. In the special case of convex MDPs, as defined in \cref{eq:cvx-prob_fenchel_dual}, FTL has a simpler form: 
\begin{equation}
\label{eq:cost_ftl}
   \cost^k = \arg\max_{\cost \in \Lambda} \sum\nolimits_{j=1}^{k-1} \g(\dpij, \cost) =   \arg\max_{\cost \in \Lambda} \left(\cost \cdot \sum\nolimits_{j=1}^{k-1}\dpij - K f^*(\cost)\right) = \nabla f(\bdpikm),
\end{equation}
where $\bdpikm=\sum\nolimits_{j=1}^{k-1}\dpij$ and the last equality follows from the fact that $(\nabla f^*)^{-1} = \nabla f$ \citep{rockafellar1970convex}.  The average regret of FTL is guaranteed to be $\regret \le c / \sqrt{K}$ under some assumptions \citep{hazan2007logarithmic}. In some cases, and specifically when the set $\mathcal{K}$ is a polytope and the function $f$ is strongly convex, FTL can enjoy logarithmic or even constant regret; see \citep{huang2016following,hazan2007logarithmic} for more details.

\textbf{Online Mirror Descent (OMD)} uses the following update \cite{nemirovskij1983problem,beck2003mirror}:
$$\cost^k = \arg\max_{\cost \in \Lambda} \left( (\cost - \cost^{k-1})\cdot \nabla_{\cost} \g(\dpikm,\cost^{k-1})+ \alpha_k B_r(\cost,\cost^{k-1})\right),$$
where $\alpha_k$ is a learning rate and $B_r$ is a Bregman divergence \citep{bregman1967relaxation}. For $B_r(x) = 0.5||x||^2_2,$ we get online gradient descent \citep{zinkevich2003online} and for $B_r(x) = x \cdot \log(x)$ we get multiplicative weights \citep{freund1997decision} as special cases. We also note that OMD is equivalent to a linearized version of Follow the Regularized Leader (FTRL) \citep{mcmahan2011follow,hazan2016introduction}. The average regret of OMD is $\regret \le c / \sqrt{K}$ under some assumptions, see, for example \citep{hazan2016introduction}.


\subsection{Policy Players}

\subsubsection{Best Response} In OCO, the best response is to simply ignore the history and play the best option on the current round, which has guaranteed average regret bound of $\regret \le 0$ (this requires knowledge of the \emph{current} loss function, which is usually not applicable but is in this case). When applied to \cref{eq:cvx-prob_fenchel_dual}, it is possible to find the best response $\dpik$ using standard RL techniques since 
$$
    \dpik = \arg \min_{\dpi \in\mathcal{K}} \g_k(\dpi, \cost^k) =   \arg \min_{\dpi \in\mathcal{K}} \dpi\cdot \cost^k - f^*(\cost^k) = \arg \max_{\dpi \in\mathcal{K}} \dpi\cdot (-\cost^k),
$$
which is an RL problem for maximizing the reward $(-\cost^k)$.
In principle, any RL algorithm that eventually solves the RL problem can be used to find the best response, which substantiates our claim in the introduction.
For example, tabular Q-learning executed for sufficiently long and with a suitable exploration strategy will converge to the optimal policy \cite{watkins1992q}. In the non-tabular case we could parameterize a deep neural network to represent the Q-values \cite{mnih2015human} and if the network has sufficient capacity then similar guarantees might hold. We make no claims on efficiency or tractability of this approach, just that in principle such an approach would provide the best-response at each iteration and therefore satisfy the required conditions to solve the convex MDP problem. 

\subsubsection{Approximate Best Response}  The caveat in using the best response as a policy player is that in practice, it can only be found approximately by executing an RL algorithm in the environment. This leads to defining an approximate best response via the  Probably Approximately Correct (PAC) framework. We say that a policy player is PAC$(\epsilon,\delta)$, if it finds an $\epsilon$-optimal policy to an RL problem with probability of at least $1-\delta$. In addition, we say that a policy $\pi'$ is $\epsilon$-optimal if its state occupancy $\dpi'$ is such that $$\max_{\dpi \in\mathcal{K}} \dpi\cdot (-\cost^k) - \dpi'\cdot (-\cost^k) \le \epsilon.$$ 

For example, the algorithm in \citep{lattimore2012pac} can find an $\epsilon$-optimal policy to the discounted RL problem after seeing $O\big(\frac{SA}{(1-\gamma)^3\epsilon^2}\log(\frac{1}{\delta})\big)$ samples; and the algorithm in \citep{jin2020efficiently} can find an $\epsilon$-optimal policy for the average reward RL problem after seeing $O\big(\frac{t_{\text{mix}}^2SA}{\epsilon^2}\log(\frac{1}{\delta})\big)$ samples, where $t_{\text{mix}}$ is the mixing time (see, \eg, \citep{levin2017markov, zahavy2019average} for a formal definition). 
The following Lemma analyzes the sample complexity of \cref{alg:meta} with an approximate best response policy player for the average reward RL problem \cite{jin2020efficiently}. The result can be easily  extended to the discounted case using the algorithm in \citep{lattimore2012pac}.
Other relaxations to the best response for specific algorithms can be found in \citep{syed2008game,miryoosefi2019reinforcement,jaggi2013revisiting,hazan2019provably}.

\begin{lemma}[The sample complexity of approximate best response in convex MDPs with average occupancy measure]
\label{thm:approx_best_response}
For a convex function $f$, running \cref{alg:meta} with an oracle cost player with regret $\regret^\cost = O(1/K)$ and an approximate best response policy player that solves the average reward RL problem in iteration $k$ to accuracy $\epsilon_k=1/k$  returns an occupancy measure $\bdpi^K$ that satisfies $f(\bdpi^K) - f^{\text{OPT}}\le \epsilon$ with probability $1-\delta$ after seeing $O(t_{\text{mix}}^2SA\log(2K/\epsilon\delta)/\epsilon^3\delta^3)$ samples. Similarly, for $\regret^\cost = O(1/\sqrt{K})$, setting $\epsilon_k=1/\sqrt{k}$ requires $O(t_{\text{mix}}^2SA\log(2K/\epsilon\delta)/\epsilon^4\delta^4)$ samples. 
\end{lemma}

\subsubsection{Non-Stationary RL Algorithms}
We now discuss a different type of policy players; instead of solving an MDP to accuracy $\epsilon$, these algorithms perform a \emph{single} RL update to the policy, with cost $-\cost_k$. In our setup the reward is known and deterministic but non-stationary, while in the standard RL setup it is unknown, stochastic, and stationary. We conjecture that any RL algorithm can be adapted to the \emph{known} non-stationary reward setup we consider here. In most cases both Bayesian \cite{osband2013more, o2018variational} and frequentist \cite{azar2017minimax, jaksch2010near} approaches to the stochastic RL problem solve a modified (\eg, by adding optimism) Bellman equation at each time period and swapping in a known but non-stationary reward is unlikely to present a problem.

To support this conjecture we shall prove that this is exactly the case for UCRL2 \citep{jaksch2010near}. UCRL2 is an RL algorithm that was designed and analyzed in the standard RL setup, and we shall show that it is easily adapted to the non-stationary but known reward setup that we require. To make this claim more general, we will also discuss a similar result for the MDPO algorithm \cite{shani2021online} that was given in a slightly different setup.

UCRL2 is a model based algorithm that maintains an estimate of the reward and the transition function as well as confidence sets about those estimates. In our case the reward at time $k$ is known, so we only need to consider uncertainty in the dynamics.
UCRL2 guarantees that in any iteration $k$, the true transition function is in a confidence set with high probability, \ie, $P \in\optset$ for confidence set $\optset$.  If we denote by $J^{P,R}_\pi$ the value of policy $\pi$ in an MDP with dynamics $P$ and reward $R$ then the optimistic policy is $\piopt = \arg\max_\pi \max_{P'\in \optset} J^{P',-\cost_k}_\pi$. Acting according to this policy is guaranteed to attain low regret. In the following results for UCRL2 we will use the constant $D,$ which denotes the diameter of the MDP, see \citep[Definition 1]{jaksch2010near} for more details. In the supplementary material (\cref{sec:ucrl_proof}), we provide a proof sketch that closely follows \citep{jaksch2010near}.

\begin{lemma}[Non stationary regret of UCRL2]
\label{lemma:ucrl2}
For an MDP with dynamics $P,$ diameter $D$, an arbitrary sequence of known and bounded rewards $\left\{ r^i: \max_{s,a} |r^i(s,a)|\le 1 \right\}_{i=1}^K,$  such that the optimal average reward at episode $k,$ with respect to $P$ and $r_k$ is $J^\star_k$, then with probability at least $1-\delta$, the average regret of UCRL2 is at most $\regret = \frac{1}{K}\sum_{k=1}^K J^\star_k - J_k^{\piopt} \le O(DS\sqrt{A\log(K/\delta)/K}).$
\end{lemma}

Next, we give a PAC$(\epsilon, \delta)$ sample complexity result for the mixed policy $\bar \pi^K,$ that is produced by running \cref{alg:meta} with UCRL2 as a policy player. 
\newpage
\begin{lemma}[The sample complexity of non-stationary RL algorithms in convex MDPs]
\label{thm:sample_rl} 
For a convex function $f,$ running \cref{alg:meta} with an oracle cost player with regret $\regret^{\cost} \le c_0/\sqrt{K}$ and UCRL2 as a policy player returns an occupancy measure $\bdpi^K$ that satisfies $f(\bdpi^K) - f^{\text{OPT}}\le \epsilon$ with probability $1-\delta$ after $K=O\left(\frac{D^2S^2A}{\delta^2\epsilon^2}\log(\frac{2DSA}{\delta\epsilon})\right)$ steps.
\end{lemma}


\textbf{MDPO.}
Another optimistic algorithm is Mirror Descent Policy Optimization \citep[MDPO]{shani2020optimistic}. MDPO is a model free RL algorithm that is very similar to popular DRL algorithms like TRPO \citep{schulman2015trust} and MPO \citep{abdolmaleki2018maximum}. In \citep{geist2019theory,shani2020adaptive,agarwal2020optimality}, the authors established the global convergence of MDPO and in \citep{cai2020provably,shani2020optimistic}, the authors showed that MDPO with optimistic exploration enjoys low regret.

The analysis for MDPO is given in a finite horizon MDP with horizon $H$, which is not the focus of our paper. Nevertheless, to support our conjecture that any stochastic RL algorithm can be adapted to the \emph{known} non-stationary reward setup, we quickly discuss the regret of MDPO in this setup. We also note that MDPO is closer to practical DRL algorithms \citep{tomar2020mirror}. In a finite horizon MDP with horizon $H$ and known, non-stationary and bounded rewards, the regret of MDPO is bounded by $\regret \le O(H^2S\sqrt{A/K})$ \citep[Lemma 4]{shani2021online} with high probability. 

To compare this result with UCRL2, we refer to a result from \cite{rosenberg2019online}, which analyzed UCRL2 in the adversarial setup, that includes our setup as a special case. In a finite horizon MDP with horizon $H$ it was shown that setting $\delta=SA/K$ with probability $1-\delta$ its regret is bounded by $\regret \le O(HS\sqrt{A\log(K)/K})$ \citep[Corollary 5]{rosenberg2019online}, which is better by a factor of $H$ than MDPO.


\textbf{Discussion.}
Comparing the results in \cref{thm:sample_rl} with \cref{thm:approx_best_response} suggests that using an RL algorithm with non stationary reward as a policy player requires O($1/\epsilon^2$) samples to find an $\epsilon-$optimal policy, while using an approximate best response requires O($1/\epsilon^3$). In first glance, this results also improves the previously best known result of Hazan et al. \cite{hazan2019provably} for approximate Frank-Wolfe (FW) that requires O($1/\epsilon^3$) samples.  
However, there are more details that have to be considered as we now discuss.

Firstly, \cref{thm:sample_rl} and \cref{thm:approx_best_response} assume access to an oracle cost player with some regret and do not consider how to implement such a cost player. The main challenge is that the cost player does not have access to the true state occupancy and must estimate it from samples. If we do not reuse samples from previous policies to estimate the state occupancy of the current policy we will require O($1/\epsilon^3$) trajectories overall \cite{hazan2019provably}. A better approach would use the samples from previous episodes to learn the transition function. Then, given the estimated transition function and the policy, we can compute an approximation of the state occupancy. We conjecture that such an approach would lead to a O($1/\epsilon^2$) sample complexity, closing the gap with standard RL. 

Secondly, while our focus is on the dependence in $\epsilon$, our bound \cref{thm:sample_rl} is not tight in $\delta$, \ie, it scales with $1/\delta^2$ where it should be possible to achieve a $\log(1/\delta)$ scaling. Again we conjecture an improvement in the bound is possible; see, \eg, \cite[Appendix F.]{kaufmann2021adaptive}.


\section{Convex Constraints}
\label{sec:constrained}
We have restricted the presentation so far to unconstrained convex problems, in this section we extend the above results to the constrained case.
The problem we consider is
\[
\min_{\dpi \in \mathcal{K}} f(\dpi) \quad \mbox{subject to}\quad g_i(\dpi) \leq 0, \quad i=1, \ldots m,
\]
where $f$ and the constraint functions $g_i$ are convex. Previous work focused on the case where both $f$ and $g_i$ are linear \citep{altman1999constrained,cmdpblog,borkar2005actor,tessler2018reward,efroni2020exploration,calian2021balancing,bhatnagar2012online}.  We can use the same Fenchel dual machinery
we developed before, but now taking into account the constraints. Consider the Lagrangian
\[
L(\dpi, \mu) = f(\dpi) + \sum\nolimits_{i=1}^m \mu_i g_i(\dpi) =  \max_\nu \left( \nu\cdot \dpi - f^*(\nu) \right) + \sum\nolimits_{i=1}^m \mu_i\max_{v_i}\left(\dpi v_i - g_i^*(v_i)\right).
\]
over dual variables $\mu \geq 0$,
with new variables $v_i$ and $\nu$. At first glance this does not look convex-concave, however we can
introduce new variables $\zeta_i = \mu_i v_i$ to obtain
\begin{equation}
  \label{eq:convex_constraint_lagrange}
  L(\dpi, \mu, \nu, \zeta_1, \ldots, \zeta_m) = \nu\cdot \dpi - f^*(\nu) + \sum\nolimits_{i=1}^m \left(\dpi \zeta_i - \mu_i g_i^*(\zeta_i / \mu_i)\right).  
\end{equation}
This is convex (indeed affine) in $\dpi$ and concave in $(\nu, \mu, \zeta_1, \ldots, \zeta_m)$, since it includes the perspective transform of the functions $g_i$ \cite{boyd2004convex}. The Lagrangian involves a cost vector, $\nu + \sum_{i=1}^m \zeta_i$, linearly interacting with $\dpi$, and therefore we can use the same policy players as before to minimize this cost. 
For the cost player, it is possible to use OMD on \cref{eq:convex_constraint_lagrange} jointly for the variables $\nu,\mu$ and $\zeta$.  It is more challenging to use best-response and FTL for the cost-player variables as the maximum value of the Lagrangian is unbounded for some values of $d_\pi$. Another option is to treat the problem as a \emph{three}-player game. In this case the policy player controls $d_\pi$ as before, one cost player chooses $(\nu, \zeta_1, \ldots, \zeta_m)$ and can use the algorithms we have previously discussed, and the other cost player chooses $\mu$ with some restrictions on their choice of algorithm.  Analyzing the regret
in that case is outside the scope of this paper.

\section{Examples}
\label{sec:examples}
In this section we explain how existing algorithms can be seen as instances of the meta-algorithm for various choices of the objective function $f$ and the cost and policy player algorithms $\costalg$ and $\rlalg$. We summarized the relationships in \cref{table:c_MDPs}. 


\subsection{Apprenticeship Learning}

In apprenticeship learning (AL), we have an MDP without an explicit reward function. Instead, an expert provides demonstrations which are used to estimate the expert state occupancy measure $d_E$. Abbeel and Ng  \cite{abbeel2004apprenticeship} formalized the AL problem as finding a policy $\pi$ whose state occupancy is close to that of the expert by minimizing the convex function $f(\dpi) = ||\dpi - d_E||$.
The convex conjugate of $f$ is given by $ f^*(y) = y\cdot d_E $ if $||y||_*\le 1$ and $\infty$ otherwise, where $||\cdot||_*$ denotes the dual norm. Plugging $f^*$ into \cref{eq:cvx-prob_fenchel_dual} results in the following game:
\begin{equation}
    \label{eq:al_minimax}
    \min_{\dpi \in \mathcal{K}}  ||\dpi - d_E|| = 
    \min_{\dpi \in \mathcal{K}}\max_{||\cost||_*\le1} \cost\cdot \dpi - \cost\cdot d_E.
\end{equation}
Inspecting 
\cref{eq:al_minimax}, we can see that the norm in the function $f$ that is used to measure the distance from the expert  induces a constraint set for the cost variable, which is a unit ball in the dual norm. 

\paragraph{$\costalg$=OMD, $\rlalg$=Best Response/RL.}

The Multiplicative Weights AL algorithms \citep[MWAL]{syed2008game} was proposed to solve the AL problem with $f(\dpi) = ||\dpi-d_E||_\infty$. It uses the best response as the policy player and multiplicative weights as the cost player (a special case of OMD). MWAL has also been used to solve AL in contextual MDPs \citep{belogolovsky2019inverse} and to find feasible solutions to convex-constrained MDPs \citep{miryoosefi2019reinforcement}. 
We note that in practice the best response can only be solved approximately, as we discussed in \cref{sec:players}. Instead, in online AL \citep{shani2021online} the authors proposed to use MDPO as the policy player, which guarantees a regret bound of $\regret\le c/\sqrt{K}$. They showed that their algorithm is equivalent to Wasserstein GAIL \citep{xiao2019wasserstein,zhang2020wasserstein} and in practice tends to perform similarly to GAIL.

\paragraph{$\costalg$=FTL, $\rlalg$=Best Response.}
When the policy player plays the best response and the cost player plays FTL, \cref{alg:meta} is equivalent to the Frank-Wolfe algorithm \citep{frank1956algorithm,abernethy2017frankwolfe} for minimizing $f$ (\cref{eq:cvx-prob}). Pseudo-code for this is included in the appendix (\cref{alg:fw}).
The algorithm finds a point $\dpik \in \mathcal{K}$ that has the largest inner-product (best response) with the negative gradient (\ie, FTL).


Abbeel and Ng  \cite{abbeel2004apprenticeship} proposed two algorithms for AL, the projection algorithm and the max margin algorithm. The projection algorithm is essentially a FW algorithm, as was suggested in the supplementary \citep{abbeel2004apprenticeship} and was later shown formally in \citep{zahavy2019apprenticeship}. Thus, it is a projection free algorithm in the sense that it avoids projecting $\dpi$ into $\mathcal{K}$, despite the name. In their case the gradient is given by $\nabla_f(\dpi) = \dpi - d_E$. Thus, finding the best response is equivalent to solving an MDP whose reward is $d_E - \dpi$.  In a similar fashion, FW can be used to solve convex MDPs more generally \citep{hazan2019provably}. Specifically, in \citep{hazan2019provably}, the authors considered the problem of pure exploration, which they defined as finding a policy that maximizes entropy.

\textbf{Fully Corrective FW.} The FW algorithm has many variants (see \citep{jaggi2013revisiting} for a survey) some of which enjoy faster rates of convergence in special cases. Concretely, when the constraint set is a polytope, which is the case for convex MDPs (\cref{def:polytope}), some variants achieve a linear rate of convergence \citep{jaggi2015global,zahavy2019apprenticeship}. 
One such variant is the Fully corrective FW, which replaces the learning rate update (see line 4 of \cref{alg:fw} in the supplementary), with a minimization problem over the convex hull of occupancy measures at the previous time-step. This is guaranteed to be at least as good as the learning rate update.
Interestingly, the second algorithm of Abbeel and Ng \citep{abbeel2004apprenticeship}, the max margin algorithm, is exactly equivalent to this fully corrective FW variant. This implies that the max-margin algorithm enjoys a better theoretical convergence rate than the `projection' variant, as was observed empirically in \citep{abbeel2004apprenticeship}. 

\subsection{GAIL and DIAYN: \texorpdfstring{$\costalg$} ==FTL, \texorpdfstring{$\rlalg$} ==RL}
\label{subsec:vic}

We now discuss the objectives of two popular algorithms, GAIL \cite{ho2016generative} and DIAYN \cite{eysenbach2018diversity}, which perform AL and diverse skill discovery respectively. Our analysis suggests that GAIL and DIAYN share the same objective function. In GAIL, this objective function is minimized, which is a convex MDP, however, in DIAYN it is maximized, which is therefore not a convex MDP. We start the discussion with DIAYN and follow with a simple construction showing the equivalence to GAIL. 

\textbf{DIAYN.} Discriminative approaches \cite{gregor2016variational,eysenbach2018diversity} rely on the intuition that skills are diverse when they are entropic and easily discriminated by observing the states that they visit. Given a probability space $(\Omega, \mathcal{F}, \mathcal{P})$, state random variables $S: \Omega \rightarrow \mathcal{S}$ and latent skills $Z : \Omega \rightarrow \mathcal{Z}$ with prior $p$, the key term of interest being maximized in DIAYN \citep{eysenbach2018diversity} is the mutual information:
\begin{align}
\label{eq:diyan}
I(S;Z) = \mathbb{E}_{z \sim p; s \sim \dpiz} [\log p (z | s)-\log p(z)],
\end{align}
where $\dpiz$ is the stationary distribution induced by the policy $\pi(a \mid s, z)$. For each skill $z$, this corresponds to a standard RL problem with (conditional) policy $\pi(a \mid s, z)$ and reward function $r(s | z) = \log p (z | s) - \log p(z)$.  The first term encourages the policy to visit states for which the underlying skill has high-probability under the posterior $p(z \mid s)$, while the second term ensures a high entropy distribution over skills. In practice, the full DIAYN objective further regularizes the learnt policy by including entropy terms $-\log \pi(a\mid s,z)$. For large state spaces, $p(z|s)$ is typically intractable and Eq.~\ref{eq:diyan} is replaced with a variational lower-bound, where the true posterior is replaced with a learned discriminator $q_\phi(z | s)$.  Here, we focus on the simple setting where $z$ is a categorical distribution over $|Z|$ outcomes, yielding $|Z|$ policies $\pi^z$, and $q_\phi$ is a classifier over these $|Z|$ skills with parameters $\phi$.

We now show that a similar intrinsic reward can be derived using the framework of convex MDPs. We start by writing the true posterior as a function of the per-skill state occupancy $\dpiz=p(s\mid z)$, and using Bayes rules, $p(z|s) = \frac{\dpiz(s)p(z)}{\sum_{k} \dpik(s)p(k)}.$ Combing this with \cref{eq:diyan} yields:
\begin{align}
    \mathbb{E}_{z \sim p(z), s\sim \dpiz} [\log p (z | s) - p(z)] 
    &= \sum_z p(z) \sum_{s} \dpiz(s) \left[ \log \left( \frac{\dpiz(s)p(z)}{\sum_{k} \dpik(s)p(k)}\right)- \log p(z)\right] \nonumber\\ 
    &= \sum_z p(z) \text{KL}(\dpiz || \sum_{k} p(k) \dpik) = \mathbb{E}_z \text{KL}(\dpiz ||  \mathbb{E}_k \dpik)
    \label{eq:pivic},
\end{align}
where KL denotes the Kullback–Leibler divergence \cite{kullback1997information}. 


Intuitively, finding a set of policies $\pi^1,\ldots,\pi^z$ that minimize \cref{eq:pivic} will result in finding policies that visit similar states, measured using the KL distance between their respective state occupancies $\dpione,\ldots, \dpiz.$ This is a convex MDP because the KL-divergence is jointly convex in both arguments \citep[Example 3.19]{boyd2004convex}. We will soon show that this is the objective of GAIL. 
On the other hand, a set of policies that maximize \cref{eq:pivic} is diverse, as the policies visit different states, measured using the KL distance between their respective state occupancies $\dpione,\ldots, \dpiz$. 

We follow on with deriving the FTL player for the convex MDP in \cref{eq:pivic}. We will then show that this FTL player is producing an intrinsic reward that is equivalent to the intrinsic reward used in GAIL and DIAYN (despite the fact that DIAYN is not a convex MDP). 
According to \cref{eq:cost_ftl}, the FTL cost player will produce a cost $\lambda^k$ at iteration $k$ given by 
\begin{align}
\label{eq:gradeintvic}
  \nabla_{\dpiz} \text{KL}(\dpiz || \sum_{k} p(k) \dpik)&= 
  p(z) \Big[ \log \frac{\dpiz}{\sum_{k} \dpik p(k)} + 1 - \frac{\dpiz p(z)}{\sum_{k} \dpik p(k)} \Big] - p(z) \frac{\sum_{j \neq z}p(j)\dpij}{\sum_{k} \dpik p(k)}
  \nonumber\\ &=p(z) \Big[ \log(p(z|s)) -\log(p(z))\Big] ,
\end{align}
where the equality follows from writing the posterior as a function of the per-skill state occupancy $\dpiz=p(s\mid z)$, and using Bayes rules, $p(z|s) = \frac{\dpiz(s)p(z)}{\sum_{k} \dpik(s)p(k)}.$ Replacing the posterior $p(z|s)$ with a learnt discriminator $q_{\phi}(z|s)$ 
recovers the mutual-information rewards of DIAYN. \footnote{In a previous version of this work we mistakenly forgot to take the gradient of the objective in \cref{eq:gradeintvic} w.r.t $\dpiz$ for the summands $j\neq z$. Taking this gradient gives the last term in the first row of \cref{eq:gradeintvic}, which in turn, cancels out the "gradient correction term" we introduced in the previous version. The current derivation retrieves the standard reward used in DIAYN and VIC.}


\textbf{GAIL.}  We further show how \cref{eq:pivic} extends to GAIL \citep{ho2016generative} via a simple construction. Consider a binary latent space of size $|Z|=2,$ where $z=1$ corresponds to the policy of the agent and $z=2$ corresponds to the policy of the expert which is fixed. In addition, consider a uniform prior over the latent variables, \ie, $p(z=1) = \frac{1}{2}.$
By removing the constant terms in \cref{eq:gradeintvic}, one retrieves the GAIL \citep{ho2016generative} algorithm. The cost $\log(p(z|s))$ is the probability of the discriminator to identify the agent, and the policy player is MDPO (which is similar to TRPO in GAIL). 

\section{Discussion}
In this work we reformulated the convex MDP problem as a convex-concave game between the agent and another player that is producing costs (negative rewards) and proposed a meta-algorithm for solving it. 

We observed that many algorithms in the literature can be interpreted as instances of the meta-algorithm by selecting different pairs of subroutines employed by the policy and cost players.
The Frank-Wolfe algorithm, which combines best response with FTL, was originally proposed for AL \cite{abbeel2004apprenticeship,zahavy2019apprenticeship} but can be used for any convex MDP problem as was suggested in \cite{hazan2019provably}. Cheung \cite{NEURIPS2019_a02ffd91} further studied the combination of FW with UCRL,  proposed a novel gradient thresholding scheme and applied it to complex vectorial objectives.  Zhang et al. \citep{zhang2020variational}, unified the problems of RL, AL, constrained MDPs with linear constraints and maximum entropy exploration under the framework of convex MDPs. We extended the framework to allow convex constraints (\cref{sec:constrained}) and explained the objective of GAIL as a convex MDP (\cref{subsec:vic}). We also discussed non convex objectives (\cref{sec:problem_formulation}) and analyzed unsupervised skill discovery via the maximization of mutual information (\cref{subsec:vic}) as a special case.
Finally, we would like to point out a recent work by Geist et al. \citep{geist2021concave}, which was published concurrently to ours, and studies the convex MDP problem from the viewpoint of mean field games. 

There are also algorithms for convex MDPs that cannot be explained as instances of \cref{alg:meta}. In particular, Zhang et al. \citep{zhang2020variational} proposed a policy gradient algorithm for convex MDPs in which each step of policy gradient involves solving a new saddle point problem (formulated using the Fenchel dual). This is different from our approach since we solve a single saddle point problem iteratively, and furthermore we have much more flexibility about which algorithms the policy player can use. Moreover, for the convergence guarantee \citep[Theorem 4.5]{zhang2020variational} to hold, the saddle point problem has to be solved exactly, while in practice it is only solved approximately \citep[Algorithm 1]{zhang2020variational}, which hinders its sample efficiency. Fenchel duality has also been used in off policy evaluation (OPE) in \citep{nachum2019dualdice,yang2020off}. The difference between these works and ours is that we train  a policy to minimize an objective, while in OPE a target policy is fixed and its value is estimated from data produced by a behaviour policy. 


In order to solve a practical convex MDP problem in a given domain it would be prudent to use an RL algorithm that is known to be high performing for the vanilla RL problem as the policy player. From the theoretical point of view this could be MDPO or UCRL2, which we have shown come with strong guarantees. From the practical point of view using a high performing DRL algorithm, which may be specific to the domain, will usually yield the best results. For the cost player using FTL, \ie, using the gradient of the objective function, is typically the best choice.

\begin{ack}
We would like to thank Yasin Abbasi-Yadkorie, Vlad Mnih, Jacob Abernethy, Lior Shani and Doina Precup for their comments and discussion on this work. Work done at DeepMind, the authors received no specific funding for this work. 
\end{ack}

\newpage
\bibliographystyle{abbrvnat}
\bibliography{ref}

\appendix

\newpage
\section{Checklist}
\begin{enumerate}
\item For all authors
\begin{enumerate}
  \item Do the main claims made in the abstract and introduction accurately reflect the paper's contributions and scope?
    \answerYes
  \item Did you describe the limitations of your work?
    \answerYes{We discussed the case where $f$ is non-convex and experimented with DIAYN as an example. We also discussed the fact that if the RL problem is a hard exploration problem, then the subroutine we use to solve it must be able to solve hard exploration problems. }
  \item Did you discuss any potential negative societal impacts of your work?
    \answerNo. This is a theoretical paper and to the best of our understanding it should not have any societal impacts. 
  \item Have you read the ethics review guidelines and ensured that your paper conforms to them?
    \answerYes
\end{enumerate}

\item If you are including theoretical results...
\begin{enumerate}
  \item Did you state the full set of assumptions of all theoretical results?
    \answerYes
	\item Did you include complete proofs of all theoretical results?
    \answerYes
\end{enumerate}

\item If you ran experiments.
We note that our experiments are a proof of concept for our approach. All of our experiments were performed with a basic agent, without hyper parameter tuning and evaluated on simple grid worlds. We are happy to share more details in case the reviewers will find something missing.  
\begin{enumerate}
  \item Did you include the code, data, and instructions needed to reproduce the main experimental results (either in the supplemental material or as a URL)?
    \answerNo
  \item Did you specify all the training details (\eg, data splits, hyperparameters, how they were chosen)?
    \answerYes
	\item Did you report error bars (\eg, with respect to the random seed after running experiments multiple times)?
    \answerYes
	\item Did you include the total amount of compute and the type of resources used (\eg, type of GPUs, internal cluster, or cloud provider)?
    \answerNo
\end{enumerate}

\item If you are using existing assets (\eg, code, data, models) or curating/releasing new assets...
\begin{enumerate}
  \item If your work uses existing assets, did you cite the creators?
    \answerYes
  \item Did you mention the license of the assets?
    \answerNo
    \item Did you include any new assets either in the supplemental material or as a URL?
    \answerNo
  \item Did you discuss whether and how consent was obtained from people whose data you're using/curating?
    \answerNo
  \item Did you discuss whether the data you are using/curating contains personally identifiable information or offensive content?
    \answerNo
\end{enumerate}

\item If you used crowdsourcing or conducted research with human subjects...
\begin{enumerate}
  \item Did you include the full text of instructions given to participants and screenshots, if applicable?
    \answerNo
  \item Did you describe any potential participant risks, with links to Institutional Review Board (IRB) approvals, if applicable?
    \answerNo
    \item Did you include the estimated hourly wage paid to participants and the total amount spent on participant compensation?
    \answerNo
    \end{enumerate}

\end{enumerate}

\newpage
\section{Proposition 1}
\label{sec:propone}
\propone*
\begin{proof}
Beginning with the discounted case, the average cost is given by
\begin{align*}
    J^\mathrm{\gamma}_\pi &=  (1-\gamma)\Expect \sum_{t=1}^\infty \gamma^t r_t \\
    &= (1-\gamma)\sum_{t=1}^\infty \sum_{s} \Prob_\pi(s_t = s) \sum_a \pi(s, a)\gamma^t r(s, a) \\
    &=  (1-\gamma)\sum_{s, a} \left(\sum_{t=1}^\infty \gamma^t \Prob_\pi(s_t = s)  \pi(s, a)\right) r(s, a) \\
    &=  \sum_{s, a} \dpi^\gamma(s, a) r(s, a).
\end{align*}
Similarly, for the average reward case
\begin{align*}
    J^\mathrm{avg}_\pi &= \lim_{T\rightarrow \infty} \frac{1}{T} \Expect\sum_{t=1}^T r_t\\
    &=\lim_{T\rightarrow \infty}\frac{1}{T} \sum_{t=1}^T \sum_{s} \Prob_\pi(s_t = s) \sum_a \pi(s, a) r(s, a) \\
    &=  \sum_{s, a} \left(\lim_{T\rightarrow \infty}\frac{1}{T} \sum_{t=1}^T \Prob_\pi(s_t = s)  \pi(s, a)\right) r(s, a) \\
    &=  \sum_{s, a} \dpi^\mathrm{avg}(s, a) r(s, a).
\end{align*}
\end{proof}

\newpage
\section{FW algorithms}

\subsection{Pseudo code}
\begin{algorithm}[H]
\caption{Frank-Wolfe algorithm}
\begin{algorithmic}[2]
\STATE \textbf{Input:} a convex and smooth function $f$
\STATE \textbf{Initialize:} Pick a random element $\dpione \in \mathcal{K}$.
\FOR{$i = 1,\ldots,T$}
    \STATE $\dpikp = \arg\max_{\pi \in \Pi} \dpi  \cdot -\nabla f(\bdpik)$
    \STATE $\bdpikp = (1-\alpha_i)  \bdpik +  \alpha_i \dpikp$
\ENDFOR
\end{algorithmic}
\caption{Frank-Wolfe algorithm}
\label{alg:fw}
\end{algorithm}

\subsection{Linear convergence}
\begin{theorem}[Linear Convergence \citep{jaggi2015global}]
\label{thm:fw_lin}
Suppose that f has L-Lipschitz gradient  and is $\mu$-strongly convex. Let $D = \{\dpi, \forall \pi \in \Pi \}$ be the set of all the state occupancy's of deterministic policies in the MDP and let $\mathcal{K} = Co(D)$ be its Convex Hull. Such that $\mathcal{K}$ a polytope with vertices $D$, and let $M= diam(\mathcal{K})$. Also, denote the Pyramidal Width of $D,$  $\delta = PWidth(D)$ as in \citep[Equation 9 1]{jaggi2015global}.

Then the suboptimality $h_t$ of the iterates of all the fully corrective FW algorithm decreases geometrically at each step, that is
$$
f(\bdpikp) \le (1 - \rho) f(\bdpik) \text{ , where } \rho = \frac{\mu \delta^2}{4L M^2}
$$
\end{theorem}

\section{Sample complexity proofs}

\begin{lemma*}[The sample complexity of non-stationary RL algorithms in convex MDPs]
For a convex function $f,$ running \cref{alg:meta} with an oracle cost player with regret $\regret^{\cost} \le c_0/\sqrt{K}$ and UCRL2 as a policy player returns a mixed policy $\bar \pi^K$ that satisfies $f(\bdpi^K) - f^{\text{OPT}}\le \epsilon$ with probability $1-\delta$ after $K=O\left(\frac{D^2S^2A}{\delta^2\epsilon^2}\log(\frac{2 DSA}{\delta\epsilon})\right)$ steps.
\end{lemma*}

\textbf{Proof}.
In \cref{thm:meta} and the discussion below it, we showed that 
$ f(\bdpi^K) - f^{\text{OPT}} \le \regret^\cost + \regret^{\pi}$.  From \cref{lemma:ucrl2}, we have that with probability $1-\delta',$ a "positive event" happens, and the regret of the UCRL2 player, $\epsilon_K,$ is upper bounded by $\regret^{\pi} = \frac{1}{K}\sum_{k=1}^K J^\star_k - J_k^{\piopt} \le c_1 DS\sqrt{A\log(K/\delta')/K})$ for some constant $c_1.$ Recall that the function $f$ has bounded gradients and therefore, the non stationary reward is upper bounded by $1.$ Thus, when the "positive event" does not happen (with probability $\delta'$), we can always upper bound the regret by $\regret^{\pi} = \frac{1}{K}\sum_{k=1}^K J^\star_k - J_k^{\piopt} \le 1.$ Using Markov's inequality, we have that 
\begin{align*}
    \text{Pr}( f(\bdpi^K) - f^{\text{OPT}} \ge \epsilon) &\le \frac{\mathbb{E} f(\bdpi^K) - f^{\text{OPT}}}{\epsilon} \le \frac{\mathbb{E} (\regret^{\pi} + \regret^{\cost})}{\epsilon}  \\ &\le \frac{1}{\epsilon} \left((1-\delta')c_1 DS\sqrt{A\log(K/\delta')/K}) + \delta'  + c_0 \sqrt{1/K}\right) \\& \le \frac{1}{\epsilon} \left((c_2 DS\sqrt{A\log(K/\delta')/K}) + \delta' \right),
\end{align*}
thus, if we choose $\delta' = \frac{\epsilon \delta}{2},$ we have that in order for
$\text{Pr}( f(\bdpi^K) - f^{\text{OPT}} \ge \epsilon)$ to be smaller than $\delta$ after $K$ steps, it is enough to find a value for $K$ such that $\frac{1}{\epsilon}c_2DS \sqrt{A\log(2 K/ \epsilon \delta)/K}\le \delta/2,$ which is achieved for $K=\frac{cD^2S^2A}{\delta^2\epsilon^2}\log(\frac{2 DSA}{\delta\epsilon}) \qed.$


\begin{lemma*}[The sample complexity of approximate best response in convex MDPs with average occupancy measure]
For a convex function $f$, running \cref{alg:meta} with an oracle cost player with regret $\regret^\cost = O(1/K)$ and an approximate best response policy player that solves the average reward RL problem in iteration $k$ to accuracy $\epsilon_k=1/k$  returns an occupancy measure $\bdpi^K$ that satisfies $f(\bdpi^K) - f^{\text{OPT}}\le \epsilon$ with probability $1-\delta$ after seeing $O(t_{\text{mix}}^2SA\log(2K/\epsilon\delta)/\epsilon^3\delta^3)$ samples. Similarly, for $\regret^\cost = O(1/\sqrt{K})$, setting $\epsilon_k=1/\sqrt{k}$ requires $O(t_{\text{mix}}^2SA\log(2K/\epsilon\delta)/\epsilon^4\delta^4)$ samples. 
\end{lemma*}

To solve an MDP to accuracy $\epsilon_k$, it is sufficient to run an RL algorithm for $O(1/\epsilon_k^2)$ iterations. This is a lower bound and an upper bound in $\epsilon$, see, for example \cite{jin2020efficiently} for an upper bound of  $O\left(\frac{t_{\text{mix}}^2SA}{\epsilon^2}\log(1/\delta)\right)$ and a lower bound \citep{jin2021towards} of $O\left(\frac{t_{\text{mix}}SA}{\epsilon^2}\log(1/\delta)\right)$ for the average reward case. 

We continue the proof using the algorithm of \cite{dann2015sample} as the approximate best response player, \ie, we invoke their algorithm at iteration $k$ to find an $\epsilon_k-$optimal solution with probability $1-\delta_k = 1-\delta'/K.$ Applying the union bound over the iterations gives us that with probability of $1-\delta'$, the regret of the policy player is  $\regret^{\pi} = \frac{1}{K} \sum \epsilon_k.$ \footnote{Note that the iterations are independent from each other from the perspective of the approximate best response player so it is possible to apply the union bound. This is because each iteration involves solving a new MDP, and the upper bound does not make any assumptions about the structure of the reward in this MDP.}

We consider two cases for the cost player. In the first, 
we will consider average regret of $\regret^\cost = c/\sqrt{K},$ which is feasible for any of the cost players we considered in this paper. In this case, we will set the per-iteration $\epsilon$ to be $\epsilon_k = c/\sqrt{k}$. We have that $$\regret =  \regret^{\pi} + \regret^\cost \le \frac{1}{K}\sum_{k=1}^K c_2/\sqrt{k} + c_1/\sqrt{K} \le \frac{1}{K} c_3 \sqrt{k} + c_1/\sqrt{K} \le c_4 / \sqrt{K}.$$
Then, via Markov inequality we get that 
\begin{align*}
    \text{Pr}( f(\bdpi^K) - f^{\text{OPT}} \le \epsilon) &\le \frac{\mathbb{E} f(\bdpi^K) - f^{\text{OPT}}}{\epsilon} \le \frac{\mathbb{E} \regret^{\pi} + \regret^{\cost}}{\epsilon}  \le  \frac{1}{\epsilon}\left(\frac{c_4}{\sqrt{K}} + \delta'  \right).
\end{align*}
Setting $\delta'=\epsilon\delta/2$ implies that it is enough to run the algorithm for $K=c_5/\epsilon^2\delta^2$ iterations to find an $\epsilon-$optimal solution with probability of $1-\delta$.

In each iteration $k,$ in order to find an $\epsilon_k = c_1/\sqrt{k}$ optimal solution w.p $1-\delta'/K,$ we need to collect $c_2kt_{\text{mix}}^2SA\log(K/\delta')$ samples. Thus, the total number of samples is $$\sum_{k=1}^{c_5/\epsilon^2\delta'^2} c_2kt_{\text{mix}}^2SA\log(K/\delta') = c_2t_{\text{mix}}^2SA\log(K/\delta') \sum_{k=1}^{c_5/\epsilon^2\delta^2} k \le c_3t_{\text{mix}}^2SA\log(2K/\epsilon\delta)/\epsilon^4\delta^4.$$

In the second scenario we have a cost player with constant regret, and therefore average regret of $\regret^\cost \le c_1/K,$ which is possible to achieve under some assumptions \citep{huang2016following}. We set $\epsilon_k = c/k$ and get that
$$\regret =  \regret^{\pi} + \regret^\cost \le \frac{1}{K}\sum_{k=1}^K c_2/k + c_1/K \le \frac{\log(K)}{K} c_3 \sqrt{k} + c_1/K \le c_4\frac{\log(K)}{K}.$$
Then, via Markov inequality we get that 
\begin{align*}
    \text{Pr}( f(\bdpi^K) - f^{\text{OPT}} \le \epsilon) &\le \frac{\mathbb{E} f(\bdpi^K) - f^{\text{OPT}}}{\epsilon} \le \frac{\mathbb{E} \regret^{\pi} + \regret^{\cost}}{\epsilon}  \le  \frac{1}{\epsilon}\left(\frac{c_4\log(K)}{K} + \delta'  \right).
\end{align*}
Setting $\delta'=\epsilon\delta/2$ implies that it is enough to run the algorithm for $K=c_5/\epsilon\delta$ iterations to find an $\epsilon-$optimal solution with probability of $1-\delta$.

In each iteration $k,$ in order to find an $\epsilon_k = c_1/k$ optimal solution w.p $1-\delta'/K,$ we need to collect $c_2k^2t_{\text{mix}}^2SA\log(K/\delta')$ samples, which leads to a total number of samples of $$\sum_{k=1}^{c_5/\epsilon\delta} c_2k^2t_{\text{mix}}^2SA\log(K/\delta') = c_2t_{\text{mix}}^2SA\log(K/\delta') \sum_{k=1}^{c_5/\epsilon\delta} k^2 \le c_3t_{\text{mix}}^2SA\log(2K/\epsilon\delta)/\epsilon^3\delta^3.\qed$$

\textbf{Discussion on how to choose the schedule for $\epsilon_k$}

The overall regret of the game is the sum of the regret of the policy player and the cost player, and the regret of the game is asymptotically 
\begin{equation}
\label{eq:regret_rl}
 \regret =  \regret^{\pi} + \regret^\cost = O\left(\max (\regret^{\pi},\regret^{\cost})\right)
\end{equation}

Consider the general case of $\epsilon_k = 1/k^p.$  Note that for the average regret $\frac{1}{K}\sum_{k=1}^K 1/k^p$ to go to zero as $K$ grows, the sum $\frac{1}{K}\sum_{k=1}^K 1/k^p$ must be smaller than $K$, so $p$ must be positive. 
In addition, for larger values of $p$, $\epsilon_k$ is smaller. Thus the regret is smaller, but at the same time, it requires more samples to solve each RL problem. 
Inspecting the maximum in \cref{eq:regret_rl}, we observe that it does not make sense to choose a value for $p$ for which $\frac{1}{K}\sum_{k=1}^K 1/k^p<\regret^{\cost},$ since it will not improve the overall regret and will require more samples, than, for example,  setting $p$ such that  $\frac{1}{K}\sum_{k=1}^K 1/k^p = \regret^{\cost}.$

Thus, in the case that the cost player has constant regret, $\regret^\cost = O(1/K),$ we set $p\in (0,1],$ and in the case that the cost player has regret of $\regret^\cost = O(1/\sqrt{K}),$ we set $p\in (0,0.5].$

We now continue and further inspect the regret. We have that $\frac{1}{K} \sum_{k=1}^K \epsilon_k = \frac{1}{K} \sum_{k=1}^K 1/k^p = O(k^{-p})$ for $p\in (0,1),$ and $\log(K)/K$ for $p=1.$ Neglecting logarithmic terms, we continued with $O(k^{-p})$ for both cases. In other words, it is sufficient to run the meta-algorithm for $K=1/\epsilon^p$ iterations to guarantee an error of at most $\epsilon$ for the convex MDP problem. 

Thus, to solve an MDP to accuracy $\epsilon_k = 1/k^p$ it requires $k^{2p}$ iterations, and the overall sample complexity is therefore  $\sum_{k=1}^{1/\epsilon^p} k^{2p} = O(1/\epsilon^{\frac{2p+1}{p}}).$ 

The function $1/\epsilon^{\frac{2p+1}{p}}$ is monotonically increasing in $p$, so it attains minimum for the highest value of $p$ which is $0.5$ or $1,$ depending on the cost player. We conclude that the optimal sample complexity with approximate best response is $O(1/\epsilon^3)$ for the cost player that has constant regret and $O(1/\epsilon^4)$ for a cost player with average regret of $\regret^\cost = O(1/\sqrt{K}).$

\section{Proof sketch for Lemma 3}
\label{sec:ucrl_proof}
We denote by $r^*_k $ the optimal average reward at time $k$ in an MDP with dynamics $P$ and reward $r_k = -\cost_k.$ We want to show that 
$$
R_k =\sum_k  r^*_k - r_k(s_k,a_k)  \le c/\sqrt{K},
$$
that is, that the total reward that the agent collects has low regret compared to the sum of optimal average rewards. 

To show that, we make two minor adaptations to the UCRL2 algorithm and then verify that its original analysis also applies to this non-stationary setup. The first modification is that the nonstatioanry version of UCRL2 uses the known reward $r_k$ at time $k$ (which in our case is the output of the cost player) instead of estimating the unknown, stochastic, stationary, extrinsic reward. Since the current reward $r_k$ is known and deterministic, there is no uncertainty about it, and we only have to deal with uncertainty with respect to to the dynamics. The second modification is that we compute a new optimistic policy (using extended value iteration) in each iteration. This optimistic policy is computed with the current reward $r_k$, and the current uncertainty set about the dynamics $\optset$. This also means that all of our episodes are of length $1.$

After making these two clarifications, we follow the proof of UCRL2 and make changes when appropriate. We note that the analysis, basically, does not require any modifications, but we repeat the relevant parts for completeness. We begin with the definition of the regret at episode $k,$ which is now just the regret at time $k: $ $$\Delta_k = \sum_{s,a}v_k(s,a)(r^*_k - r_k(s,a)),$$
where $v_k(s,a)$ in our case is an indicator on the state action pair $s_k,a_k$, and $R_k = \sum_k \Delta_k$.

The instantaneous regret $\Delta_k$ measures the difference between the optimal average reward $r^*_k$, with respect to reward $r_k,$ and the reward $r_k(s,a)$ that the agent collected at time $k$ by visiting state $s$ and taking action $k$ from the reward that is produced by the cost player. 

Section 4.1 in the UCRL2 paper is the first step in the analysis. It bounds possible fluctuations in the random reward. This step is not required in our case since our reward at time $k$ is the output of the cost player, which is known in all the states and deterministic.  

Section 4.2 considers the regret that is caused by failing confidence regions, that is, the event that the true dynamics and true reward are not in the confidence region. In our case there is only confidence region for the dynamics (since the reward is known), which we denote by $\optset$. Summing the expected regret from episodes in which $P \notin \optset$ results in a $\sqrt{K}$ term in the regret, 
$$\Delta_k \le \sum_{s,a}v_k(s,a) (r^*_k - r_k(s,a)) + \sqrt{K},$$
where from now on, we continue with the event that $P \in \optset.$

Next, we denote the optimistic policy and optimistic MDP as the solution of the following problem $\piopt, \popt = \arg\max_{\pi\in\Pi, P'\in \optset}J^{P',r_k}_\pi.$ In addition, we denote by $\ropt$ the optimstic average reward, that is, the average reward of the policy $\piopt$ in the MDP with the optimstic dynamics $\popt$ and reward $r_k.$ We also note that $\piopt$ is the optimal average reward policy in this MDP by its definition.

We now continue with the case that $P \in \optset.$ The next step is to bound the difference between the optimal average reward $r^*_k$ and the optimistic average reward $\ropt.$  We note that both $\ropt$ and $r^*_k$ are average rewards that correspond to $r_k.$ The difference between them is that $r^*_k$ is the optimal average reward in an MDP with the true dynamics $P$ and $\ropt$ is the optimal average reward in an MDP with the optimistic dynamics $\popt.$ Thus, the fact that the reward is known, in our case, does not change the fact that that the optimstic reward is a function of the dynamics uncertainty set $\optset$.

To compute the optimstic policy and dynamics, UCRL2 uses the extended value iteration procedure of \cite{strehl2008analysis} to efficiently compute the following iterations:

\begin{align}
\label{eq:evi}
    u_0(s) &= 0 \\\nonumber
    u_{i+1}(s) &= \max_{a\in A} \left\{ r_k(s,a) + \max_{P\in \popt} \sum_{s'\in S}P(s'|s,a) u_i(s')\right\},
\end{align}

Using Theorem 7 from \citep{jaksch2010near} we have that running extended value iteration to find the optimistic policy in the optimistic MDP for $t_k$ iterations guarantees that  $\tilde r_k \ge r^*_k -1/\sqrt{t_k}.$ Thus, we have that:

\begin{align*}
  \Delta_k &\le \sum_{s,a}v_k(s,a) (r^*_k - r_k(s,a))  + \sqrt{K} \le \sum_{s,a}v_k(s,a) (\tilde r_k - r_k(s,a)) + 1/\sqrt{t_k}  + \sqrt{K}\\
\end{align*}
Using \cref{eq:evi}, we write the last iteration of the extended value iteration procedure as:
\begin{equation}
\label{eq:ui1}
    u_{i+1}(s) = r_k(s_k,\piopt(s)) + \sum_{s'\in S}\popt(s'|s,(\piopt(s))) u_i(s')
\end{equation}

Theorem 7 from \citep{jaksch2010near} guarantees that after running extended value iteration for $t_k$ we have that 
\begin{equation}
\label{eq:ui2}
\|u_{i+1}(s)-u_i(s) -\tilde r_k\| \le 1/ \sqrt{t_k}.
\end{equation}

Plugging \cref{eq:ui1} in \cref{eq:ui2} we have that:
\begin{align}
\label{eq:ui3}
\|r_k(s_k,\piopt(s)) -\tilde r_k + \sum_{s'\in S}\popt(s'|s,(\piopt)) u_i(s') -u_i(s) \|  \le 1/ \sqrt{t_k},
\end{align}
and therefore 
$$
\tilde r_k - r_k(s_k,a_k) =  \tilde r_k - r_k(s_k,\piopt(s)) \le v_k(\tilde P_k - I) u_i +1/ \sqrt{t_k}.
$$
In the next step in the proof, the vector $u_i$ is replaced with $w_k,$ which is later upper bounded by the diameter of the MDP $D.$ To conclude, we have that 
\begin{align*}
\Delta_k &\le \sum_{s,a}v_k(s,a)(\tilde r_k - r_k(s,a)) + 1/\sqrt{t_k} + \sqrt{K}  \le v_k(\tilde P_k - I) w_k + 2/\sqrt{t_k}  + \sqrt{K}.
\end{align*}
From this point on, the proof follows by bounding the term $v_k(\tilde P_k - I) w_k,$ which is only related to the dynamics, and combines all of the previous results into the final result, thus, it is possible to follow the original proof without any modification. Since the leading terms in the original proof come from uncertainty about the dynamics, we obtain the same bound as in the original paper. 

\section{Experiments: Entropy constrained RL.}
\label{sec:experiments}
Above, we presented a principled approach to using standard RL algorithms to solve convex MDPs. We also suggested that DRL agents can use this principle and solve convex MDPs by optimizing the reward from the cost player. We now demonstrate this by performing experiments with Impala \citep{impala18a}, a distributed actor-critic DRL algorithm. Our main message is that in domains where Impala can solve RL problems (\eg, problems without hard exploration), it can also solve convex MDPs.

Here we focus on an MDP with a convex constraint, where the goal is to maximize the extrinsic reward provided by the environment with the constraint that the entropy of the state-action occupancy measure must be bounded below. In other words, the agent must solve
$\max_{\dpi \in \mathcal{K}} \sum_{s,a} r(s,a) \dpi(s,a)$ subject to $H(\dpi) \ge C$, where $H$ denotes entropy and $C>0$ is a constant. The policy that maximizes the entropy over the MDP acts to visit each state as close to uniformly often as is feasible. So, a solution to this convex MDP is a policy that, loosely speaking, maximizes the extrinsic reward under the constraint that it explores the state space sufficiently. The presence of the constraint means that this is not a standard RL problem in the form of \cref{eq:rl-prob}. However, the agent can solve this problem using the techniques developed in this paper, in particular those discussed in \cref{sec:constrained}.
 
We evaluated the approach on the bsuite environment `Deep Sea', which is a hard exploration problem where the agent must take the exact right sequence of actions to discover the sole positive reward in the environment; more details can be found in \cite{osband2019behaviour}. In this domain, the features are one-hot state features, and we estimate $\dpi$ by counting the state visitations. For these experiments we chose $C$ to be half the maximum possible entropy for the environment, which we can compute at the start of the experiment and hold fixed thereafter.
We equipped the agent with the (non-stationary) Impala algorithm, and the cost-player used FTL. We present
the results in Figure \ref{fig:deep_sea} where we compare the basic Impala agent, the entropy-constrained Impala agent and bootstrapped DQN \cite{osband2016deep}. As made clear in \cite{o2020making} algorithms that do not properly account for uncertainty cannot in general solve hard exploration problems. This explains why vanilla Impala, considered a strong baseline, has such poor performance on this problem. Bootstrapped DQN accounts for uncertainty via an ensemble, and consequently has good performance. Surprisingly, the entropy regularized Impala agent performs approximately as well as bootstrapped DQN, despite not handling uncertainty. This suggests that the entropy constrained approach, solved using \cref{alg:meta}, can be a reasonably good heuristic in hard exploration problems.

\begin{figure}[h]
\centering
    \includegraphics[width=0.3\linewidth]{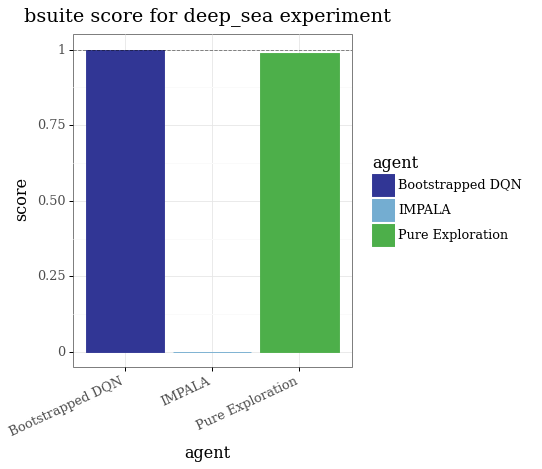}
    \includegraphics[width=0.6\linewidth]{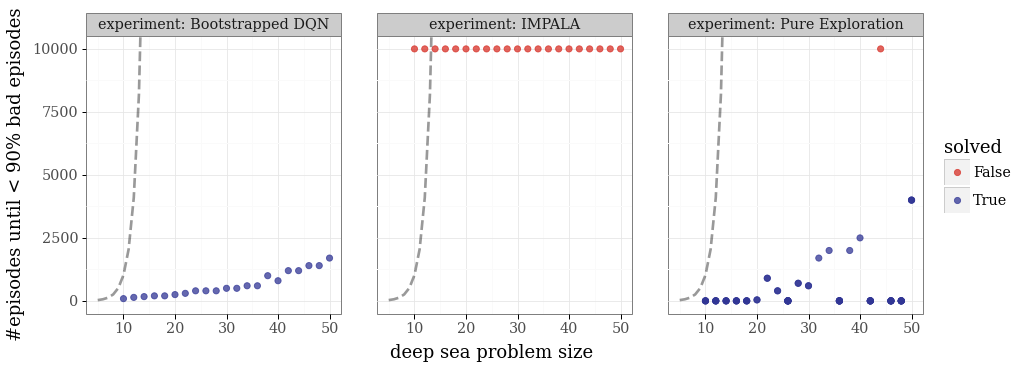}
 \caption{Entropy constrained RL}
 \label{fig:deep_sea}
\end{figure}

\end{document}